\documentclass[runningheads]{llncs}
\usepackage{tcolorbox}
\usepackage{color, colortbl}
\definecolor{Gray}{gray}{0.9}
\definecolor{LightCyan}{rgb}{0.88,1,1}
\definecolor{LightMagenta}{rgb}{1,0.88,1}
\usepackage{times}
\usepackage{tikz}
\usepackage{wrapfig}
\usetikzlibrary{decorations.pathreplacing}
\usepackage{mathtools}
\usepackage{amsmath}
\usepackage{relsize}
\usepackage{grffile}
\usepackage{fixltx2e}
\usepackage{color}
\usepackage[hyphens]{url}
\definecolor{MidnightBlue}{rgb}{0.058, 0.188, 0.578}
\usepackage{nicefrac}

\newcommand{\calx}{\mathcal{X}}
\newcommand{\caly}{\mathcal{Y}}
\DeclareMathOperator*{\argmin}{{\rm argmin}}
\DeclareMathOperator*{\argmax}{{\rm argmax}}

\newcommand{\revision}[1]{{{#1}}}

\usepackage{graphicx}
\begin{document}
\title{Feature selection in machine learning:\\ R\'enyi min-entropy vs Shannon entropy}
%
%
\author{Catuscia Palamidessi\inst{1} \and
Marco Romanelli\inst{1,2}}
\authorrunning{C. Palamidessi and M. Romanelli}
%
\institute{Inria and \'Ecole Polytechnique, France \and
Universit\`a di Siena, Italy
}
\maketitle              
\begin{abstract}
Feature selection, in the context of machine learning,  is the process of 
separating the highly predictive feature from those that might
be  irrelevant or redundant.
Information theory has been recognized as a useful concept for this task, as 
the prediction power stems from the correlation, i.e., the mutual information, between features and labels. 
Many algorithms for feature selection in the literature have adopted the Shannon-entropy-based mutual information.  
In this paper, we explore the  possibility of using R\'enyi min-entropy  instead. In particular, we propose an algorithm based on a notion of conditional 
R\'enyi  min-entropy that has been  recently adopted in the field of security and privacy, and which  is strictly related to the Bayes error. 
We prove that in general the two approaches are incomparable, in the sense that we show that we can construct  datasets on which the 
R\'enyi-based algorithm performs better than the corresponding Shannon-based one, and datasets on which the situation is reversed.
In practice, however, when considering datasets of real data, it seems  that the  R\'enyi-based algorithm tends to outperform the other one. 
We have effectuate several experiments on the BASEHOCK, SEMEION, and GISETTE datasets, and in all of them we have indeed observed that 
the R\'enyi-based algorithm gives better results.

\end{abstract}
%
%
%


\section{Introduction}\label{sec:introduction}
Machine learning has made huge  advances in recent years and is having an increasing impact on many aspects of everyday life,  as well as on  industry, science and medicine. Its power, with respect to traditional programming, relies on the capability of acquiring knowledge from experience, and more specifically, learning from data samples.

Machine learning has actually been around for quite some time: the term was coined by Arthur Samuel in 1959. The reason for the recent rapid expansion is primarily due to 
the huge amount of data that are being collected and made available, and the increased computing power, accessible at an affordable price,  to process these data. 

As the size of available datasets is becoming larger, both in terms of samples and in terms of number of features of the  data, it becomes more important to keep the dimensionality of the data under control and to  identify  the ``best'' features  on which to focus  the learning process. This is crucial to avoid an explosion of  the training complexity,
 improve the accuracy of the prediction,  and provide a better understanding of the model. Several papers in the literature of machine learning have considered this problem, including 
 \cite{Bennesar:15:ESA,Brown:12:JMLR,Cai:18:Neurocomputing,Einicke:18:JBHI,Guyon:03:MLR,Jain:00:TPAMI,Liu:05:TKDE,Liu:18:KBS,Nakariyakul:18:KBS,Sheikhpour:17:PR,Vergara:14:NCA}, to mention a few.
 
The known methods for reducing  the dimensionality can be divided in two categories: those which transform the feature space by  reshaping
the original features into new ones (\emph{feature extraction}), and those which select a subset of the features (\emph{feature selection}). 
The second category can in turn be divided in three groups: the \emph{wrapper},  the \emph{embedded},  and the \emph{filter} methods. 
The last group has the advantage of being classifier-independent, more robust with respect to the risk of overfitting, 
and more amenable to a principled approach. In particular, several proposals for  feature selection  
have successfully  applied concepts and techniques from information theory 
\cite{Battiti:94:TNN,Bennesar:15:ESA,Brown:12:JMLR,Fleuret:04:JMLR,Peng:05:TPAMI,Vergara:14:NCA,Yang:99:AIDA}. 

In this paper we focus on the filter method for classifiers, namely for machines that are trained to classify samples on the basis of their features. 
A typical example in the medical world is a predictor of the type of illness (class) given a set of symptoms (features). 
Another example in image recognition is a machine identifying a person (class) given the physical characteristics (features) visible in a picture. 
In this context, the information-theoretic  approaches to feature selection are based on the idea that  
the larger   the correlation between the selected set of features and the classes is, 
the more the classification task is likely to be correct. 
The problem of feature selection corresponds therefore to identifying a set of features as small as  possible, 
whose mutual information with the classes is above a certain threshold. 
Equivalently, since the entropy of the classes is fixed, the goal can also be  formulated in terms of the conditional entropy (aka residual entropy) of the classification given the set of features. 
Note that such residual entropy represents the \emph{uncertainty} on the correct classification of a sample once we know  the values of its selected features. Hence the goal is to
select  the smallest set of features that reduce the uncertainty of the classification to an acceptable level. 

More formally, the problem of feature selection can be stated as follows: given  random variables $F$ and $C$, modeling respectively a set of \emph{features} and a set of \emph{classes}\footnote{When clear from the context, we will use the same notation to represent both the random variable and its supporting set.}, 
  find a minimum-size subset $S \subseteq F$ such that the conditional entropy of $C$ given $S$ is below a certain threshold. Namely:
\begin{equation}\label{eq:treshold}
S = \argmin\limits_{S'}\, \{|S'|  \;  \mid \; {S'\subseteq F}\; \mbox{ and } \; H(C \mid S') \leq h\}
\end{equation}
where $h$ is the given threshold, $|S'|$  is the number of elements of $S'$, and $H(C \mid S')$ is the conditional  Shannon entropy of $C$ given $F$.

All the information-theoretic approaches to feature selection that have been proposed are based, as far as we know, on Shannon entropy, 
with the notable exception of \cite{Endo:13:CIARP} that considered the R\'enyi  entropies $H_\alpha$, where $\alpha$ is a parameter ranging over all the positive reals and $\infty$.  
In this paper we explore the particular case  of $H_\infty$,  called (\emph{R\'enyi}) \emph{min-entropy},  we develop an approach to feature selection based on min-entropy,   and we compare it with the one based on Shannon entropy. 
Our approach and analysis actually depart significantly from  \cite{Endo:13:CIARP}; the differences with that work will be explained in Section~\ref{sec:related}.

The starting point for an approach based on the min-entropy is, naturally, to replace the conditional  entropy in \eqref{eq:treshold} by  conditional min-entropy. 
Now, R\'enyi did not define the conditional version of his entropies, but there have been various proposals for it, 
in particular those by Arimoto \cite{Arimoto:75:CMJB}, Sibson \cite{Sibson:69:ZWG},  Csisz\'ar \cite{Csiszar:95:TIT}, and Cachin  \cite{Cachin:97:PhD}.
The variant that we consider here is the one by Arimoto \cite{Arimoto:75:CMJB}, which in the case $\alpha=\infty$
 has recently  become popular in security thanks to Geoffrey Smith, who  showed that it corresponds to the operational model of one-try  attack \cite{Smith:09:FOSSACS}.
More specifically,  Arimoto / Smith conditional min entropy captures the (converse of) the probability of error of a rational   attacker who knows the probability distributions and tries to infer a secret from some correlated observables. ``Rational'' here means that the attacker will try to minimize the expected probability of error,  by selecting the secret with highest posterior probability. 
Note the similarity with the classification problem, where the machine  chooses a class (secret) on the basis of the  features (observables), trying to minimize the expected probability of  classification error (misclassification).
It is therefore natural to investigate the potentiality of this notion in the  context of feature selection.
Note that, since we assume that the  attacker is rational and knows the probability distributions, the attacker is the Bayes attacker and, correspondingly, the classifier is the (ideal) Bayes classifier. The probability of misclassification is therefore the Bayes error.


By replacing the Shannon entropy $H$ with the  R\'enyi min-entropy $H_\infty$, the problem described in \eqref{eq:treshold} becomes:
\begin{equation}\label{eq:treshold-min}
S = \argmin_{S'}\, \{|S'|  \;  \mid \; {S'\subseteq F}\; \mbox{ and } \; H_\infty(C \mid S') \leq h\}
\end{equation}
where  $H_\infty(C \mid S')$ is the conditional min-entropy of C given $F$. 
Because of the correspondence between $H_\infty(\cdot \mid \cdot)$     and the  Bayes error, we can interpret \eqref{eq:treshold-min}   as stating that $S$ is the minimal set of features  for which the (ideal) Bayes classifier achieves the desired level of accuracy. 

\subsection{Contribution}
The contribution of this paper is the following:

\begin{itemize}
\item We formalize an approach to feature selection based on R\'enyi min-entropy.
\item We show that the problem of selecting the optimal set of features w.r.t. min-entropy, namely the  $S$ that satisfies \eqref{eq:treshold-min},  is NP-hard.
\item We propose an iterative greedy strategy of linear complexity  to approximate the optimal subset of features w.r.t. min-entropy. This strategy starts from the empty set, and adds a new feature at each step until we achieve the desired level of accuracy. 
\item We show that our strategy is \emph{locally optimal}, namely, at every step the new set of features is the optimal one among those that can be obtained from the previous one  by adding only one feature. (This does not imply, however, that the final result is \emph{globally optimal}.)
\item We compare our approach with that based on Shannon entropy, and we prove a negative result: neither of the two approaches is better than the other \emph{in all cases}.  
\item We compare the two approaches experimentally,  using the BASEHOCK, SEMEION, and GISETTE datasets. Despite the above incomparability result, the R\'enyi-based algorithm turns out to give better results in all these experiments.

\end{itemize}

\subsection{Plan of the paper}
In next section we recall some preliminary notions about information theory. 
In Section~\ref{sec:problem} we formulate the problem of feature-minimization and prove that it is NP-hard. 
In Section~\ref{sec:algorithm} we propose a linear greedy algorithm to approximate the solution, and we compare it with the analogous formulation in terms of Shannon entropy.
In Section~\ref{sec:ml} we show evaluations of our and Shannon-based algorithms on various datasets. 
In Section~\ref{sec:related} we discuss related work.
Section~\ref{sec:conclusion} concludes. 

\section{Preliminaries}\label{sec:preliminaries}

In this section we briefly review some basic notions from probability and information theory. We refer to \cite{Cover:91:BOOK} for more details.

Let $X, Y$ be discrete random variables with respectively $n$ and $m$ possible values: $\calx=\{x_1,x_2,\ldots,x_n\}$ and 
$\caly=\{y_1,y_2,\ldots,y_m\}$. 
Let $p_X(\cdot)$ and $p_Y(\cdot)$  indicate the probability distribution associated to $X$ and $Y$ respectively, 
and let  $p_{Y,X}(\cdot,\cdot)$  and $p_{Y|X}(\cdot|\cdot)$ indicate the joint and the conditional probability distributions, 
respectively. Namely,  $p_{Y,X}(x,y)$ represents the probability that $X=x$ and $Y=y$, while 
$p_{Y|X}(y|x)$ represents the probability that $Y=y$ given that  $X=x$. 
For simplicity, when  clear from the context, we will omit the subscript, and write for instance $p(x)$ instead  of $p_X(x)$. 

Conditional and joint probabilities are related by the chain rule 
$p(x,y)=p(x)\, p(y|x)$,
from which (by the commutativity of $p(x,y)$) we can derive the Bayes theorem: 
\begin{equation*}\label{eqn:bt}
p(x |y ) = \nicefrac{p(y |x )\,p(x )}{p(y )}.
\end{equation*}

The {\em R\'enyi entropies} (\cite{Renyi:61:Berkeley}) are a family of functions representing the uncertainty associated to a random variable.
Each  R\'enyi entropy is characterized by a non-negative real number  $\alpha$ (order), with $\alpha \neq 1$, and   is defined as
\begin{equation}\label{eqn:res}
H_\alpha(X) \stackrel{\rm def}{=} \frac{1}{1-\alpha}\log \Bigg(\sum_{i}  p(x_i)^\alpha \Bigg).
\end{equation}
If $p(\cdot)$ is uniform then all the R\'enyi entropies  are equal to $\log |X|$. 
Otherwise they are weakly decreasing in $\alpha$.
Shannon and R\'enyi min-entropy are   particular cases:
\[
\renewcommand{\arraystretch}{1.5}
\begin{array}{rll}
\alpha \rightarrow 1\qquad&H_1 (X) = - \sum_{x}  p(x)\log p(x) \quad &\mbox{Shannon entropy}\\ 
\alpha \rightarrow \infty\qquad&  H_\infty (X) = - \log \max_{x} p(x)&\mbox{min-entropy}
\end{array}
\]

Let $H_1(X,Y)$ represent  the \revision{joint entropy}  $X$ and $Y$.  
{\em  Shannon conditional entropy} of $X$ given $Y$ is the average residual entropy of $X$ once  $Y$ is known, and it is defined as
\begin{equation}\label{eqn:ce}
H_1(Y|X) \stackrel{\rm def}{=} 
 - \sum_y p(y) H_1 (X | Y = y)= 
\sum_{xy}  p(x,y)\log p(x|y) = H_1(X,Y) - H_1(Y).
\end{equation}
{\em Shannon mutual information} of $X$ and $Y$ represents the correlation of information between $X$ and $Y$, and it is defined as
\begin{equation}\label{eqn:mi}
 I_1(X;Y) \stackrel{\rm def}{=} H_1(X) - H_1(X|Y) =  H_1(X) + H_1(Y) - H_1(X,Y).
  \end{equation}
It is possible to show that $ I_1(X;Y) \geq 0$, with $ I_1(X;Y) = 0$ iff $X$ and $Y$ are independent, and that $ I_1(X;Y)= I_1(Y;X)$. 
Finally,  {\em Shannon conditional mutual information} is defined as:
\begin{equation}\label{eqn:shancmi}
I_1(X;Y|Z) \stackrel{\rm def}{=} H_1(X|Z) - H_1(X|Y,Z),
 \end{equation}
\\

\noindent
As for {\em R\'enyi  conditional min-entropy}, we use the version of \cite{Smith:09:FOSSACS}: 
\begin{equation}\label{eqn:ce}
H_\infty(X|Y) \stackrel{\rm def}{=}  -\log\sum_{y}  \max_x (p(y|x) p(x)).
\end{equation}
This definition closely corresponds to  the Bayes error, i.e., the expected error when we try to guess  $X$ once we know  $Y$, formally defined as
\begin{equation}\label{eqn:br}
{\it {\cal B}(X|Y)} \stackrel{\rm def}{=} 1 - \sum_{y} p(y) \,\max_x p(x|y).
 \end{equation}
{\em R\'enyi mutual information} is  defined as: 
\begin{equation}\label{eqn:mimin}
I_\infty(X;Y) \stackrel{\rm def}{=} H_\infty(X) - H_\infty(X|Y).
\end{equation}
It is possible to show that $ I_\infty(X;Y) \geq 0$, and that $ I_\infty(X;Y) = 0$ if  $X$ and $Y$ are independent (the reverse is not necessarily true). Contrary to Shannon mutual information,   $ I_\infty$ is not symmetric.
{\em R\'enyi conditional mutual information} is defined as 
\begin{equation}\label{eqn:mimin}
I_\infty(X;Y|Z) \stackrel{\rm def}{=} H_\infty(X|Z) - H_\infty(X|Y,Z).
 \end{equation}

\section{Formulation of the problem and its complexity}\label{sec:problem}
In this section we state formally the problem of finding a minimal set of features  that  satisfies a given bound on the classification's accuracy, and then we   show that the problem is NP-hard. 
More precisely, we are interested in finding a minimal set with respect to which the posterior R\'enyi min-entropy of the classification is bounded by a given value. We recall that the posterior R\'enyi min-entropy   is equivalent to the Bayes classification error. 

The corresponding problem for Shannon entropy is well studied in the literature of feature selection, and its NP-hardness   is a folk theorem in the area.  
However for the sake of comparing it with the case of R\'enyi min-entropy, we   restate it here in the same terms as for the latter.   

In the following, $F$ stands for the set of all features, and $C$ for the random variable that takes value in the set of classes.

\begin{definition}[\sc Min-set]
Let $h$ be a   non-negative real. The minimal-set problems for Shannon entropy and for R\'enyi min-entropy are defined as the problems of   determining the set of features $S$ such that 
\[ \begin{array}{lll}
\mbox{\rm Shannon } \quad\;\;&\mbox{\sc Min-Set}_1 \quad\;\;& 
S = \argmin\limits_{S'}\, \{|S'|  \;  \mid \; {S'\subseteq F}\; \mbox{ and } \; H_1(C \mid S') \leq h\},\\[1ex]
\mbox{{\rm R\'enyi }} \quad&\mbox{\sc Min-Set}_\infty \quad& 
S = \argmin\limits_{S'}\, \{|S'|  \;  \mid \; {S'\subseteq F}\; \mbox{ and } \; H_\infty(C \mid S') \leq h\}.
\end{array}
\]
\end{definition}

We now show that  the above problems are NP-hard

\begin{theorem}
Both   {\sc Min-Set}$_1$ and {\sc Min-Set}$_\infty$ are NP-hard.
\end{theorem}
\begin{proof}
Consider the following decisional problem {\sc Min-Features}: Let 
 $X$ be a set of examples, each of which is composed of a a binary
value specifying the value of the class and a vector
of binary values specifying the values of the  features.
Given a  number $n$, determine whether or not there exists
some feature set $S$ such that:
\begin{itemize}
\item
$S$ is a subset of the set of all input features.
\item
$S$  has cardinality $n$.
\item
There exists no two examples in $X$ that have identical values for all the features in $S$ but have different class values. 
\end{itemize}
In \cite{Davies:94:PSU} it is shown that 
{\sc Min-Features} is NP-hard by reducing 
to it the {\sc Vertex-Cover} problem, which is known to be NP-complete~\cite{Karp:72:CCC}.
We recall that the {\sc Vertex-Cover} problem 
problem may be stated as the following question:
given a graph $G$ with vertices $V$ and edges $E$, is there a
subset $V'$
of $V$, of size $m$, such that each edge in $E$ is
connected to at least one vertex in $V'$?

To complete the proof, it is sufficient to show that we can reduce {\sc Min-Features}
to  {\sc Min-Set}$_1$ and {\sc Min-Set}$_\infty$. 
Set $h=0$, and let 
\[S=\argmin\limits_{S'}\, \{|S'|  \;  \mid \; {S'\subseteq F}\; \mbox{ and } \; H_\alpha(C \mid S') = 0\},\] where $\alpha=1$ or $\alpha=\infty$.
Note that for both values of $\alpha$, $H_\alpha(C \mid S) = 0$ means that the uncertainty about $C$ is $0$ once we know the value of all features in $S$, and this is possible only if 
there exists no two examples in that have identical values for all the features in $S$ but have different class values. 
Hence to answer {\sc Min-Features} it is sufficient to check whether $|S|\leq m$ or $|S|> m$.
\qed
\end{proof}

Given that the problem is NP-hard, there is no ``efficient'' algorithm (unless P = NP) for computing exactly  the minimal set of features $S$ satisfying the bound on the accuracy. 
It is however possible to compute efficiently an approximation of it, as we will see in next section, where 
 we  propose a linear ``greedy'' algorithm  which computes an approximation of the minimal $S$. 

\section{Our proposed algorithm}\label{sec:algorithm}
Let $F$ be the set of features at our disposal, and let $C$ be random variable ranging on the set of classes. 
 Our algorithm is based on forward feature selection and  dependency maximization:
it constructs a monotonically increasing sequence  $\{S^t\}_{t\geq 0}$ of subsets of $F$, and,   
at each step, the subset  $S^{t+1}$ is obtained from $S^{t}$ by adding the next feature in order of importance (i.e., the informative contribution to classification),
taking into account the information already provided by $S^t$. The measure of the ``order of importance'' is based on conditional min-entropy.  The construction of the sequence is assumed to be done interactively with a test on  the accuracy achieved by  the current subset, using one or more classifiers. This test will provide the stopping condition: once we obtain the desired level of accuracy, the algorithm stops and gives  as result the current subset $S^T$. 
Of course, achieving a level of accuracy $1-\varepsilon$ is only possible if ${\cal B}(C \mid F) \leq \varepsilon$.
\\
\begin{definition}\label{def:agorithm}
The series $\{S^t\}_{t\geq 0}$ and $\{f^t\}_{t\geq 1}$ are inductively defined as follows:  
\[
\begin{array}{rcl}
S^{0} &\stackrel{\rm def}{=} & \emptyset\\
f^{t+1} &\stackrel{\rm def}{=} &\argmin_{f \in F\setminus S^t} H_\infty(C  \mid  f, S^t)\\
S^{t+1} &\stackrel{\rm def}{=}&  S^t\cup \{f^{t+1}\}\\
\end{array}
\]
\end{definition}

The algorithms in \cite{Brown:12:JMLR} and \cite{Vergara:14:NCA} are analogous, except that they use Shannon entropy. They also 
define $f^{t+1}$ based on the   maximization of mutual information instead   of  the minimization of conditional entropy, but this is irrelevant. In fact
$I_1 (C ; f  \mid  S^t) = H_1(C  \mid  S^t) - H_1(C  \mid  f, S^t)$,
hence maximizing $I_1 (C ; f  \mid  S^t)$ with respect to $f$ is the same as minimizing $H_1(C  \mid  f, S^t)$ with respect to $f$.

Our algorithm is locally optimal, in the sense stated by the following proposition.
\begin{proposition}\label{prop:one_step}
At every step, the set $S^{t+1}$ minimizes the Bayes error of the classification among those which are of the form $S^{t}\cup\{f\}$, namely:
\[
\forall f\in F\;\;{\cal B}(C \mid S^{t+1}) \leq  {\cal B}(C \mid S^{t}\cup \{f\} ) 
\]
\end{proposition}
\begin{proof}
Let $\vec{v},v, v'$ represent generic value tuples and values of $S^{t}$,   $f$ and $f^{t+1}$, respectively. Let $c$ represent the generic value of $C$. By definition, $H_\infty(C  \mid S^{t+1}) \leq H_\infty(C  \mid S^{t}\cup \{f\})$, for every $f\in F$. From \eqref{eqn:ce}  we then obtain 
\[
\sum_{\vec{v},v}  \max_c (p( \vec{v},v | c) p(c)) \leq \sum_{\vec{v},v'}  \max_c (p(\vec{v},v' | c) p(c) )
\]
Using the Bayes theorem~\eqref{eqn:bt}, we get
\[
\sum_{\vec{v},v}  p( \vec{v},v) \max_c p( c | \vec{v} )  \leq \sum_{\vec{v},v'} p(\vec{v},v') \max_c p(\vec{v},v' | c)   
\]
Then, from  the definition \eqref{eqn:br} we deduce 
\[
{\cal B}(C \mid S^{t}\cup \{f^{t+1}\} ) \leq  {\cal B}(C \mid S^{t}\cup \{f\} ) 
\]
\qed
\end{proof}

In the \revision{following sections we analyze some extended examples} to illustrate how the algorithm works, and also compare it  with the ones   of \cite{Brown:12:JMLR} and \cite{Vergara:14:NCA}.
\subsection{\revision{An example in which R\'enyi min-entropy gives a better feature selection than Shannon entropy}}
Let us consider the dataset in Fig.~\ref{fig:exe1}, containing ten records labeled each  by a different class, and characterized by six features (columns $f_{1}$, \ldots, $f_{5}$).
We note that   $f_{0}$ separates the classes in two sets of four and six elements respectively, while all the other columns are characterized by having two values, each of which  univocally identify one class, while the third value is associated to all the remaining classes.  For instance, in column $f_{1}$ value A univocally identifies the record of class $0$,  B univocally identifies the record of class $1$, and all the other records have the same value along that column, i.e. C. 

The last five features combined are necessary and sufficient ton completely identify  all classes, without the need of the first one. 
Note of the last five features can be replaced by $f_{0}$ for this purpose. In fact,  each pair of  records which are separated by one of the features $f_{1}$, \ldots, $f_{5}$, have the same value in column $f_{0}$.

\begin{figure} 
\begin{center}
\newcolumntype{g}{>{\columncolor{LightMagenta}}c}
\begin{tabular}{|g|c|c|c|c|c|c|}
\hline
\rowcolor{LightCyan}
\; Class       \;       &\; $f_{0}$ \;&\; $f_{1}$ \;&\; $f_{2}$ \;&\; $f_{3}$ \;&\; $f_{4}$  \;&\; $f_{5}$\;\\
\hline
0 & A & C & F & I & L & O\\
\hline
1 & A & D & F & I & L & O\\
\hline
2 & A & E & G & I & L & O\\
\hline
3 & A & E & H & I & L & O\\
\hline
4 & B & E & F & J & L & O\\
\hline
5 & B & E & F & K & L & O\\
\hline
6 & B & E & F & I & M & O\\
\hline
7 & B & E & F & I & N & O\\
\hline
8 & B & E & F & I & L & P\\
\hline
9 & B & E & F & I & L & Q\\
\hline
\end{tabular}
\caption{The dataset}\label{fig:exe1}
\end{center}
\end{figure}

If we apply the discussed feature selection method and we look for the feature that minimizes $H(C|f_{i})$ for $i \in \{0, \ldots ,5\}$ we  obtain that:
\begin{itemize}
\item The first feature selected with Shannon   is $f_{0}$, in fact $H_1(C|f_{0}) \approx 2.35$ and $H_1(C|f_{\ne 0}) = 2.4$.  (The notation $f_{\ne 0}$ stands for any of the $f_i$'s except $f_0$.)
In general, indeed, with Shannon entropy the method tends to choose a feature which splits the Classes in a way as balanced as possible. The situation after the selection of the feature $f_{0}$ is shown in Fig.~\ref{fig:first_feature}(a). 
\item The first feature selected with R\'enyi min-entropy is either $f_{1}$  or $f_{2}$ or $f_{3}$ or $f_{4}$ or $f_{5}$, in fact $H_{\infty}(C|f_{0}) \approx 2.32$ and $H_{\infty}(C|f_{\ne 0}) \approx 1.74$. In general, indeed, with R\'enyi min-entropy the method tends to choose a feature which divides the classes in as many sets as possible. The situation after the selection of  $f_{1}$ is shown in Fig.~\ref{fig:first_feature}(b). 
\end{itemize} 

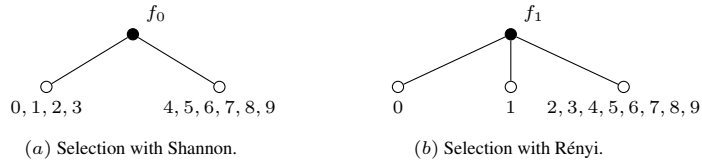
\begin{figure*}[t!]
    \centering
\begin{tikzpicture}[scale=1,font=\scriptsize][baseline=(current axis.outer east)]
\tikzstyle{solid node}=[circle,draw,inner sep=1.5,fill=black]
\tikzstyle{hollow node}=[circle,draw,inner sep=1.5]
\tikzstyle{level 1}=[level distance=7mm,sibling distance=2.3cm]
\node(0)[solid node,label=above right:{$f_{0}$}]{}
child{node[hollow node,label={[align=center]below:$ 0, 1, 2 , 3 $}]{}}
child{node[hollow node,label={[align=center]below:$ 4,  5, 6, 7, 8, 9$}]{}};
\node [below=1.3cm, align=flush center ]  { $(a)$ Selection with Shannon.  };
 \end{tikzpicture}
\qquad\qquad
\begin{tikzpicture}[scale=1,font=\scriptsize]
\tikzstyle{solid node}=[circle,draw,inner sep=1.5,fill=black]
\tikzstyle{hollow node}=[circle,draw,inner sep=1.5]
\tikzstyle{level 1}=[level distance=7mm,sibling distance=1.5cm]
\node(0)[solid node,label=above right:{$f_{1}$}]{}
child{node[hollow node,label=below:{$  0 $}]{}}
child{node[hollow node,label=below:{$  1 $}]{}}
child{node[hollow node,label={[align=center]below:$ 2, 3, 4, 5, 6, 7, 8, 9$}]{}};
\node [below=1.3cm, align=flush center ]  { $(b)$ Selection with R\'enyi. };
\end{tikzpicture}
    \caption{Classes separation  after the selection of the first feature.}\label{fig:first_feature}

\end{figure*}

\begin{figure*}[t!]
    \centering
\begin{tikzpicture}[scale=.6,font=\tiny][baseline=(current axis.outer east)]
\tikzstyle{solid node}=[circle,draw,inner sep=1.5,fill=black]
\tikzstyle{hollow node}=[circle,draw,inner sep=1.5]
\tikzstyle{level 1}=[level distance=10mm,sibling distance=1.5cm]
\node(0)[solid node,label=above right:{$f_{1}$}]{}
child{node[hollow node,label=below:{$ 0$}]{}}
child{node[hollow node,label=below:{$ 1$}]{}}
child{node[solid node,label=above right:{$f_{2}$}]{}
child{node[hollow node,label={[align=center]below:  $2$}]{}}
child{node[hollow node,label={[align=center]below:  $3$}]{}}
child{node[hollow node,label={[align=center]below:$4, 5, 6$\\$7, 8,  9$}]{}}};
\node [below=2cm, align=flush center ]  { $H_{\infty}(C|f_{1}f_{2}) = 1$.  };
\end{tikzpicture}
\quad
\begin{tikzpicture}[scale=.6,font=\tiny]
\tikzstyle{solid node}=[circle,draw,inner sep=1.5,fill=black]
\tikzstyle{hollow node}=[circle,draw,inner sep=1.5]
\tikzstyle{level 1}=[level distance=10mm,sibling distance=1.5cm]
\node(0)[solid node,label=above right:{$f_{1}$}]{}
child{node[hollow node,label=below:{$  0$}]{}}
child{node[hollow node,label=below:{$ 1$}]{}}
child{node[solid node,label=above right:{$f_{0}$}]{}
child{node[hollow node,label={[align=center]below: $2, 3$\quad}]{}}
child{node[hollow node,label={[align=center]below: $ 4, 5, 6$\\$7, 8,  9$}]{}}};
\node [below=2cm, align=flush center ]  {$H_{\infty}(C|f_{1}f_{0}) \approx 1.32$. };
\end{tikzpicture}
\quad
\begin{tikzpicture}[scale=.6,font=\tiny]
\tikzstyle{solid node}=[circle,draw,inner sep=1.5,fill=black]
\tikzstyle{hollow node}=[circle,draw,inner sep=1.5]
\tikzstyle{level 1}=[level distance=10mm,sibling distance=1.5cm]
\node(0)[solid node,label=above right:{$f_{4}$}]{}
child{node[hollow node,label=below:{$ 0$}]{}}
child{node[hollow node,label=below:{$ 1$}]{}}
child{node[solid node,label=above right:{$f_{0}$}]{}
child{node[hollow node,label={[align=center]below:$0, 1$\\$ 2, 3 $}]{}}
child{node[hollow node,label={[align=center]below:$6, 7$\\$8,  9$}]{}}};
\node [below=2cm, align=flush center]  {$H_{\infty}(C|f_{4}f_{0}) \approx 1.32$. };
\end{tikzpicture}
    \caption{Selection of the second feature with R\'enyi.}\label{fig:second_feature}
\end{figure*}
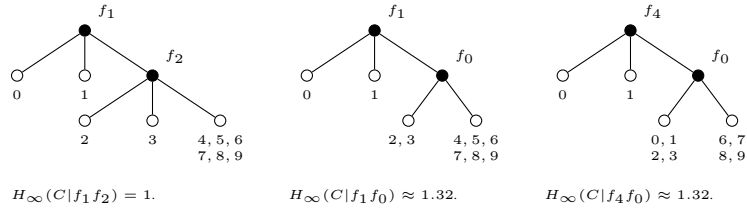
Going ahead with the algorithm,  with  Shannon entropy we will select one by one all the other features, and as already discussed we will need  all of them to completely identify all classes. 
Hence at the end the method with Shannon entropy will return all the six features (to achieve perfect classification).     
On the other hand, with R\'enyi min entropy we will select all the remaining features except $f_{0}$ to obtain the perfect discrimination. 
In fact, at any stage the selection of $f_0$ would allow to split the remaining classes in at most two sets, while any other feature not yet considered will split the remaining classes
 in three sets.  As already hinted, with R\'enyi we choose the feature that allows to split the remaining classes in the highest number of sets, hence we never select $f_0$.
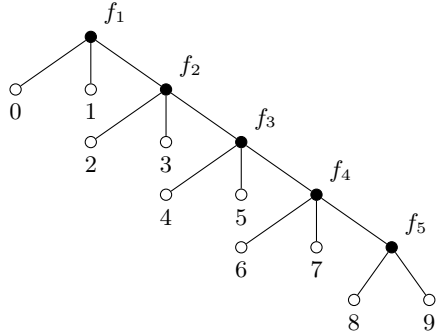
\begin{figure}
    \centering
\begin{tikzpicture}[font=\footnotesize]
\tikzstyle{solid node}=[circle,draw,inner sep=1.5,fill=black]
\tikzstyle{hollow node}=[circle,draw,inner sep=1.5]
\tikzstyle{level 1}=[level distance=7mm,sibling distance=1cm] 
\node(0)[solid node,label=above right:{$f_{1}$}]{}
child{node[hollow node,label=below:{$ 0$}]{}}
child{node[hollow node,label=below:{$ 1$}]{}}
child{node[solid node,label=above right:{$f_{2}$}]{}
child{node[hollow node,label=below:{$ 2$}]{}}
child{node[hollow node,label=below:{$ 3$}]{}}
child{node[solid node,label=above right:{$f_{3}$}]{}
child{node[hollow node,label=below:{$ 4$}]{}}
child{node[hollow node,label=below:{$ 5$}]{}}
child{node[solid node,label=above right:{$f_{4}$}]{}
child{node[hollow node,label=below:{$  6$}]{}}
child{node[hollow node,label=below:{$  7$}]{}}
child{node[solid node,label=above right:{$f_{5}$}]{}
child{node[hollow node,label=below:{$  8$}]{}}
child{node[hollow node,label=below:{$ 9$}]{}}
}}}
};
\end{tikzpicture}
    \caption{Sequence of class splitting with R\'enyi.}\label{fig:all_sequence}
\end{figure}
For instance, if we have already selected $f_{1}$, we have $H_{\infty}(C|f_{1}f_{0}) \approx 1.32$ while $H_{\infty}(C|f_{1}f_{\ne 0}) = 1$. If we have already selected $f_{4}$, we have $H_{\infty}(C|f_{4}f_{0}) \approx 1.32$ while $H_{\infty}(C|f_{4}f_{\ne 0}) = 1$. See Fig.~\ref{fig:second_feature}. 

At the end, the selection of features using R\'enyi entropy will determine the progressive splitting represented in Fig.~\ref{fig:all_sequence}.  The order of selection is not important: this particular example is conceived so that the features
 $f_{1}$, \ldots, $f_{5}$ can be selected in any order, the residual entropy is always the same. \\[-2ex]
 
\noindent
{\bf Discussion}
 It is easy to see that, in this example,  the algorithm based on R\'enyi min-entropy gives a better result not only at the end, but also at each step of the process. 
Namely, at step $t$ (cfr. Definition~\ref{def:agorithm}) the set $S^t$ of features selected with R\'enyi min-entropy gives a better classification (i.e., more accurate) than the set ${S'}^t$ that would be selected using Shannon entropy. 
More precisely, we have ${\cal B}(C \mid S^{t}) <  {\cal B}(C \mid {S'}^{t})$. 
In fact, as discussed above the set ${S'}^t$ contains necessarily the feature $f_0$, while $S^t$ does not. Let $S^{t-1}$ be the set of features selected at previous step with  R\'enyi min-entropy, and $f^{t}$ the feature selected at step $t$ (namely, $S^{t-1} = S^{t}\setminus\{f^t\}$). As argued above, the order of selection of the features $f_{1}$, \ldots, $f_{5}$ is irrelevant, hence we have ${\cal B}(C \mid S^{t-1}) = {\cal B}(C \mid {S'}^{t}\setminus \{f_0\})$ and the algorithm \emph{could} equivalently have selected ${S'}^{t}\setminus \{f_0\}$. As argued above, the next feature to be selected, with  R\'enyi, must be different from $f_0$. Hence by Proposition~\ref{prop:one_step}, and by the fact that the order of selection of $f_{1}$, \ldots, $f_{5}$ is irrelevant,  we have:  ${\cal B}(C \mid S^{t}) = {\cal B}(C \mid ({S'}^{t}\setminus \{f_0\})\cup \{f^t\}) <  {\cal B}(C \mid {S'}^{t})$. 

As a general observation, we can see  that  the method with Shannon tends to select the feature  that divides the classes in sets (one for each value of the feature) as balanced as possible, while our method tends to select the feature that divides the classes in as many sets as possible, regardless of the sets being balanced or not.  In general, both Shannon-based  and R\'enyi-based methods try to minimize the height of the tree representing the process of the splitting of the classes, but the first does it by trying to produce a tree \emph{as balanced as possible}, while the second one tries to do it by producing a tree \emph{as wide as possible}. 
Which of the method is best, it depends on the correlation of the features. Shannon works better when there are enough uncorrelated (or not much correlated) features, so that  the tree can be kept balanced while being constructed.  Next section shows an example of such situation. R\'enyi, on the contrary, is not so sensitive to correlation and can work well also when the features are highly correlated, as it was the case in the example of this section. 

The experimental results in Section~\ref{sec:ml} show that, at least in the cases we have considered, our method  outperforms the one based  on Shannon entropy. In general however the two methods are incomparable, and perhaps a good practice would be to construct both sequences at the same time, so to obtain the best result of the two.

\subsection{\revision{An example in which  Shannon entropy may give a better feature selection than R\'enyi min-entropy}}
Consider a dataset containing samples equally distributed among 32 classes, indexed from 0 to 31. 
Assume that the data have $8$ features divided in $2$ types  $F$ and $F'$, each of which consisting of $4$ features:
  $F=\{f_1,f_2,f_3,f_4\}$ and $F'=\{f'_1,f'_2,f'_3,f'_4\}$. 
\revision{The relation between the features and the classes is represented in Fig.~\ref{fig:feauresFandFprime}.}

\begin{figure}[t]
\begin{center}
\includegraphics[width=0.30\textwidth]{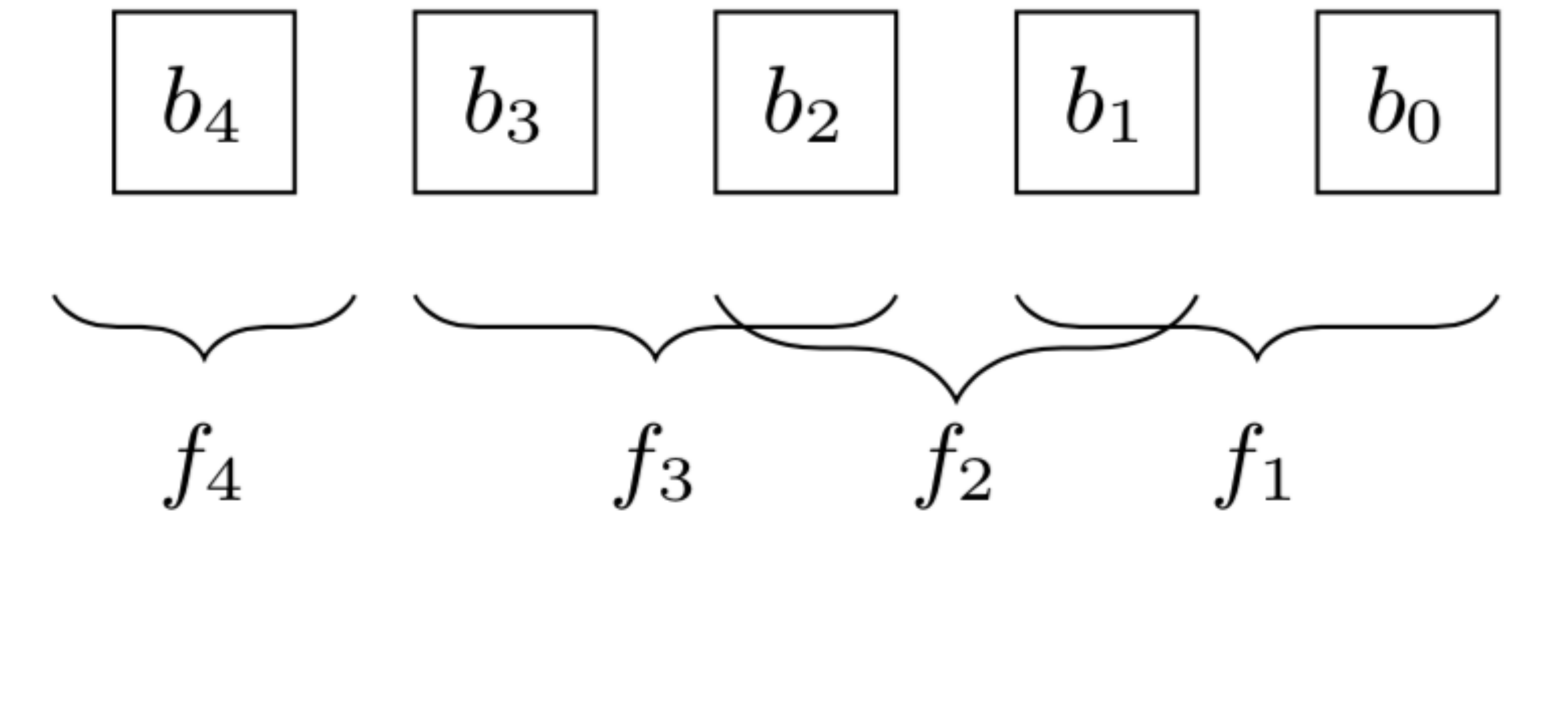} 
\qquad
\includegraphics[width=0.60\textwidth]{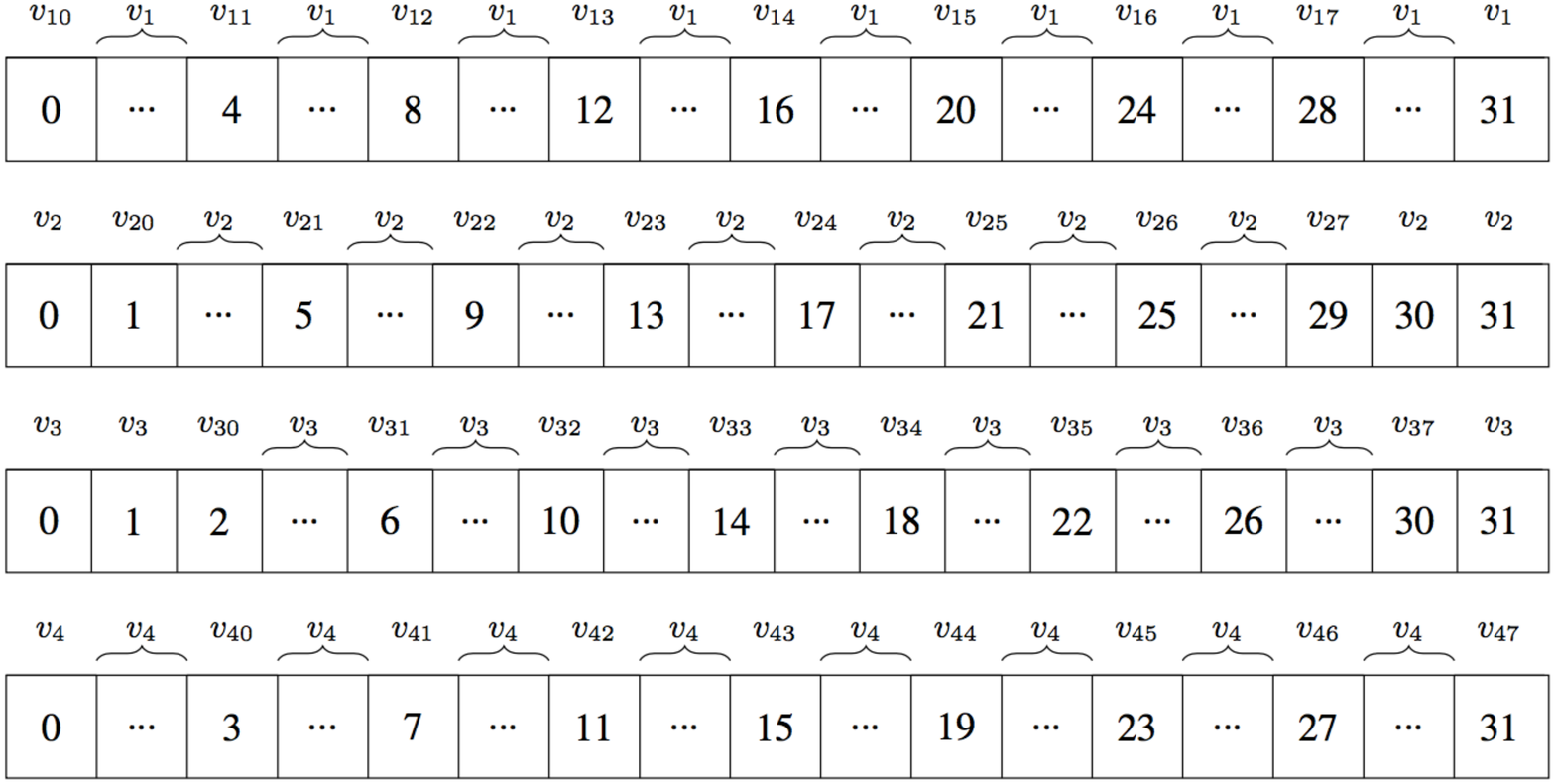} 
\end{center}
\vspace{-10mm}
\caption{Features $F$ (left)  and $F'$ (right).}\label{fig:feauresFandFprime}
\end{figure}

\revision{Because of space restriction we have omitted the computations, the interested reader can find them in the report version of this paper \cite{Palamidessi:18:REPORT}. 
At step $3$ one of the possible outcomes of the algorithm based on Shannon is the set of features $S_1^3= \{f_1,f_3,f_4\}$, and one of the possible outcomes of  the algorithm based on R\'enyi is  $S_\infty^3= \{f'_1, f'_2,f'_i\}$ where  $i$ can be, equivalently, $3$ or $4$. }
At this point the method with Shannon can stop, since the residual Shannon entropy of the classification is $H_1(C\mid S_1^3)=0$, and also the Bayes error is ${\cal B}(C\mid S_1^3)=0$, which is the optimal situation in the sense that the classification is completely accurate.
$S_\infty^3$ on the contrary does not contain enough features to give a completely accurate classification, for that we have to make a further step. 
We can see that  $S_\infty^4=F'$, and finally we have  $H_\infty(C\mid S_\infty^4)=0$.

Thus in this particular example we have that for small values of the threshold on the accuracy our method gives better results. On the other hand, if we want to achieve perfect accuracy (threshold $0$) Shannon gives better results.  

\section{Evaluation}\label{sec:ml}
In this section we evaluate   the method for feature selection that we have proposed, and we compare it with the one based on Shannon entropy by
\cite{Brown:12:JMLR} and \cite{Vergara:14:NCA}. 

To evaluate the effect of feature selection, some classification methods have to be trained and tested on the selected data. 
We  used two different methods   to avoid the dependency of the  result  on a particular algorithm.
We chose two widely used classifiers: the Support Vector Machines (SVM) and the Artificial Neural Networks (ANN).

Even though the two methods are very different, they have in common that  their efficiency is highly dependent on the choice of certain parameters. Therefore, it is worth spending some effort to identify the best values. Furthermore, we should  take  into account that the particular paradigm of SVM we chose only needs 2 parameters to be set, while for ANN the number of parameters increases (at least 4).

It is very important to choose values  as robust as possible for the parameters. 
It goes without saying that the strategy used to pick the best parameter setting should be the same  for both  Shannon entropy and  R\'enyi min-entropy.
On the other hand for  SVM and ANN we used two different hyper-parameter tuning algorithms, given that the number and the nature of the parameters to be tuned   for those classifiers is different.

In the case of SVM we tuned the cost parameter of the objective function for margin maximization (\textit{C-SVM}) and the  parameter which models the shape of the RBF kernel's bell curve (\textit{$\gamma$}).
Grid-search and Random-search are quite time demanding algorithms for the hyper-parameter tuning task but they're also widely used and referenced in literature when it comes to SVM. Following the guidelines in \cite{Chang:11:ACM} and \cite{Pedregosa:11:JML}, we decided to use Grid-search, which is quite suitable when we have to deal with only two parameters. In particular we performed Grid-search  including a 10 folds CV step.

Things are different with ANN  because many more parameters are involved and some of them  change the topology of the network itself.
Among the various strategies to attack this problem we picked Bayesian Optimization  \cite{Snoek:12:NIPS}. This  algorithm  combines steps of extensive search for a limited number of settings before inferring via Gaussian Processes (GP) which is the best setting to try next (with respect to the mean and variance and compared to the  best result obtained in the last iteration of the algorithm).
In particular we tried to fit the best model by optimizing the following parameters:
\begin{itemize}
\item number of hidden layers
\item number of hidden neurons in each layer
\item learning rate for the gradient descent algorithm 
\item size of batches to update the weight on network connections
\item number of learning epochs
\end{itemize}

To this purpose, we included in the pipeline of our code the \textit{Spearmint} Bayesian optimization codebase. Spearmint, whose theoretical bases are explained in \cite{Snoek:12:NIPS}, calls 
repeatedly an objective function to be optimized. In our case the objective function contained some \textit{tensorflow} machine learning code which run  a 10 folds CV over a dataset and the objective was to maximize the accuracy of validation. The idea was to obtain a model  able to generalize as much as possible using only the selected features before testing on a dataset which had never been seen before. 

We had to decide the stopping  criterion, which is not provided by 
\textit{Spearmint} itself. For the sake of simplicity we decided to run it for a time lapse which has empirically been proven to be sufficient in order to obtain  results  meaningful for comparison. 
A possible improvement  would be to keep running the same test (with the same number of features) for a certain amount of time without resetting the computation history of the package and only stop testing a particular configuration if the same results is output as the best for $k$ iterations in a row (for a given $k$).

Another factor, not directly connected to the different performances obtained with different entropies, but which is important for the optimization of ANN, is the choice of the activation functions for the layers of neurons. In our work we have used   ReLU    for all layers because it is well known that it works well for this aim, it is  easy to compute (the only operation involved is the max) and it avoids the sigmoid saturation issue.

\subsection{Experiments}
As already stated, at the $i$-th step of the feature selection algorithm  we consider all the features which have already been selected in the previous $i-1$ step(s). For the sake of limiting the execution time, 
we decided to consider only the first 50 selected features with both metrics.
We tried our pipeline on the following datasets:  

\begin{itemize}
\item BASEHOCK dataset: 1993 instances, 4862 features, 2 classes. This dataset has been obtained from the 20 newsgroup original dataset.
\item SEMEION dataset: 1593 instances, 256 features, 10 classes. This is a dataset with encoding of hand written characters.
\item GISETTE dataset: 6000 instances, 5000 features, 2 classes. This is the discretized version of the NIPS 2003 dataset which can be downloaded from the site of Professor Gavin Brown, Manchester University.
\end{itemize}

We implemented a bootstrap procedure (5 iterations on each dataset) to shuffle data and make  sure that the results do not depend on the particular split between training, validation and test set. 
 Each one of the 5 bootstrap iterations is a new and unrelated experimental run.  For each one of them a different training-test sets split was taken into account. Features were selected analyzing the training set  (the test set has never been taken into account for this part of the work). After the feature selection was executed according to both Shannon and R\'enyi min-entropy, we considered all the selected features adding one at each time.  So, for each bootstrap iteration we had 50 steps, and in each step we added one of the selected features, we performed hyper-parameter tuning with 10 folds CV, we trained the model with the best parameters on the whole training set and we tested it on the test set (which the model had never seen so far). This   procedure was performed both for SVM and ANN.

%
\begin{figure}[t]
\centering
\includegraphics[width=130pt, height=130pt,page=4]{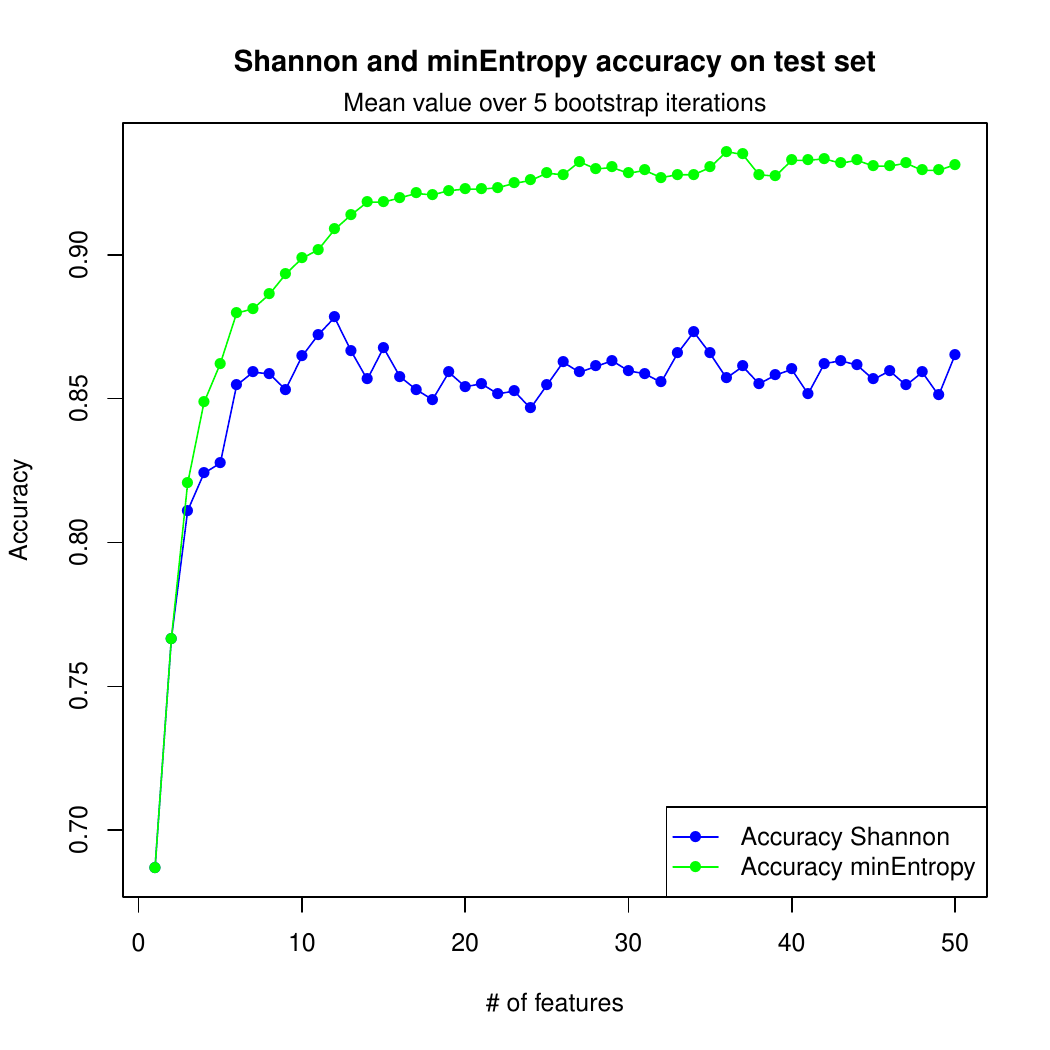}
\includegraphics[width=130pt, height=130pt,page=4]{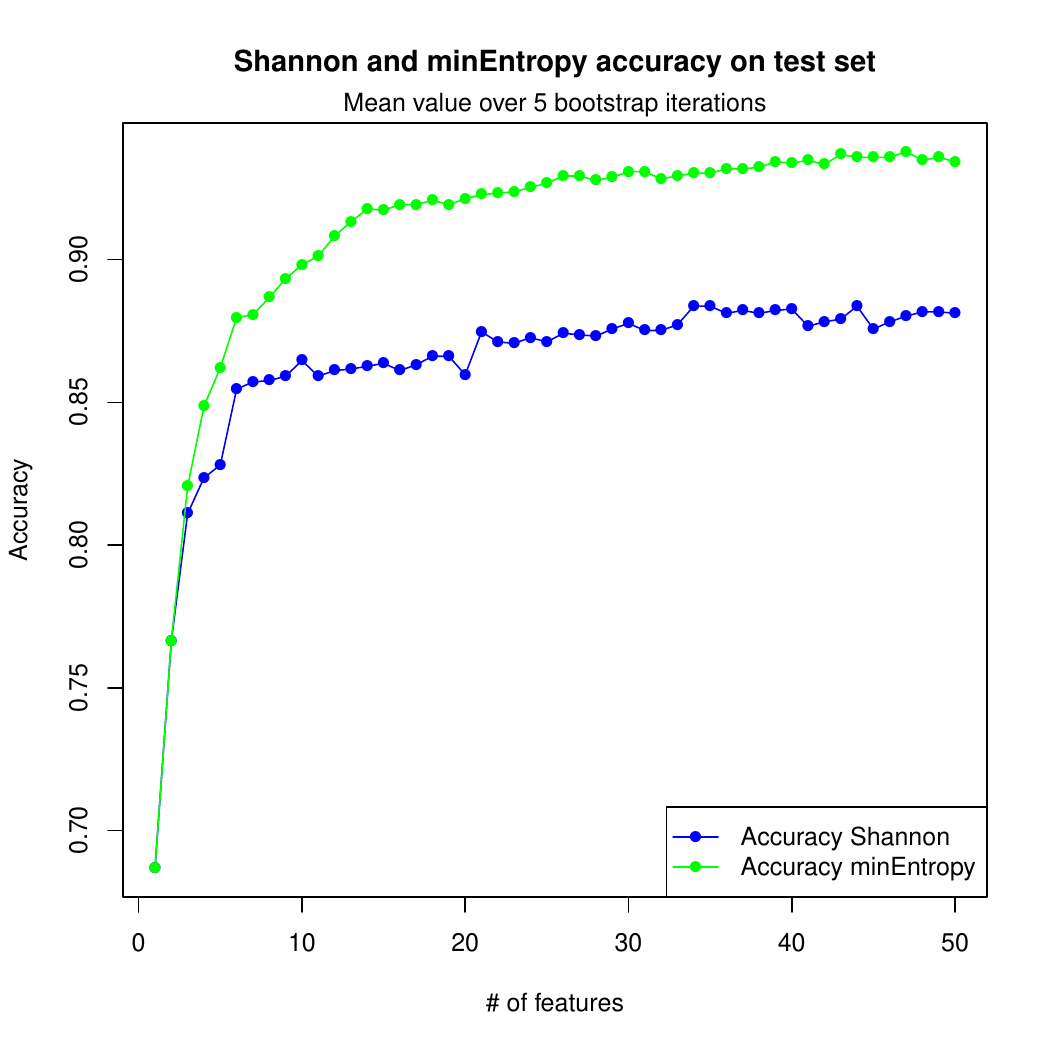}
\vspace{-3mm}
\caption{\small Accuracy of the ANN and SVM classifiers on the BASEHOCK dataset}\label{fig:bh}
\vspace{-3mm}
\end{figure}
\begin{figure}[t]
\centering
\includegraphics[width=130pt, height=130pt,page=4]{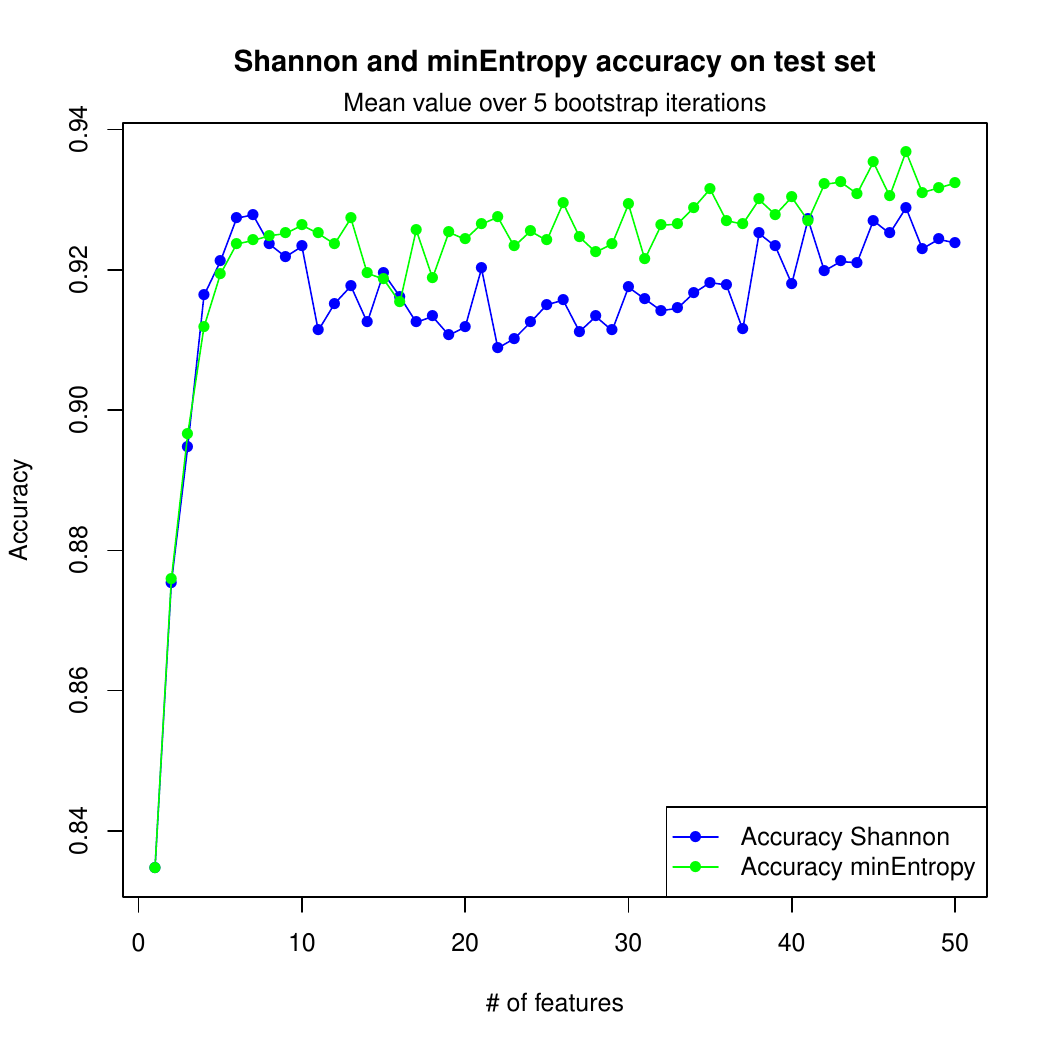}
\includegraphics[width=130pt, height=130pt,page=4]{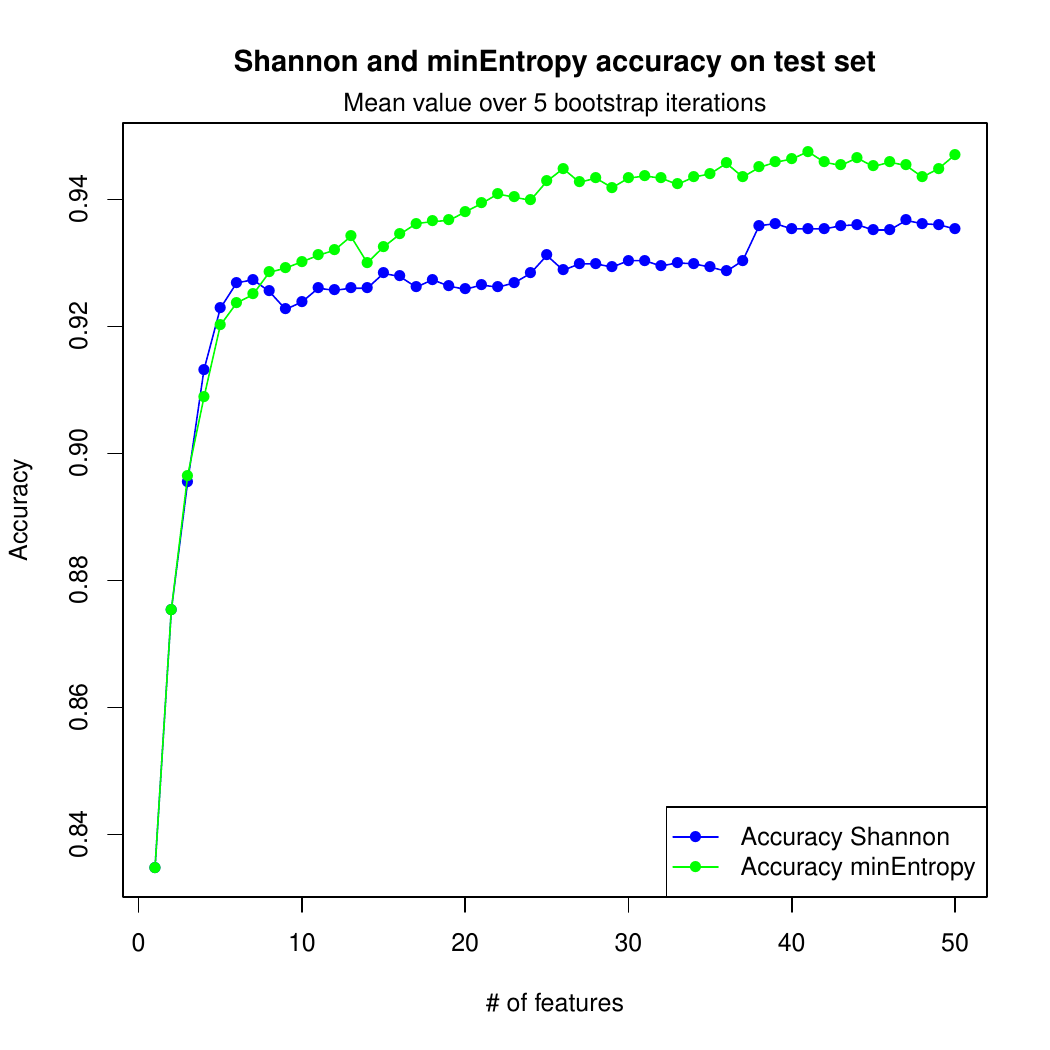}
\vspace{-3mm}
\caption{\small Accuracy of the ANN and SVM classifiers on the GISETTE dataset}\label{fig:gi}
\vspace{-3mm}
\end{figure}

We computed the average performances over the 5 iterations and  the results are in Figures~\ref{fig:bh}, \ref{fig:gi}, and \ref{fig:se}. In all cases the feature selection method using R\'enyi min-entropy usually gave better results than Shannon, especially with the  BASEHOCK dataset.

\begin{figure}[t]
\centering
\includegraphics[width=130pt, height=130pt,page=4]{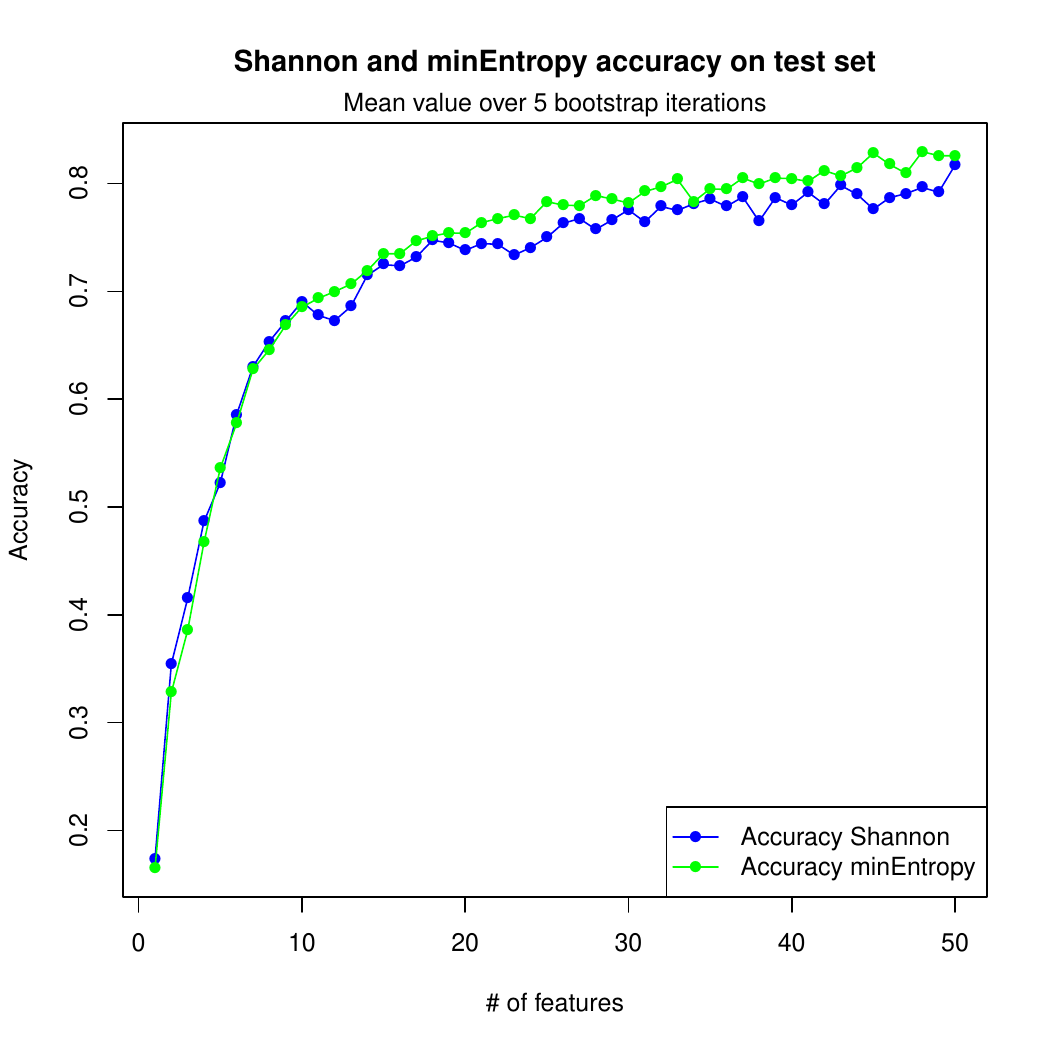}
\includegraphics[width=130pt, height=130pt,page=4]{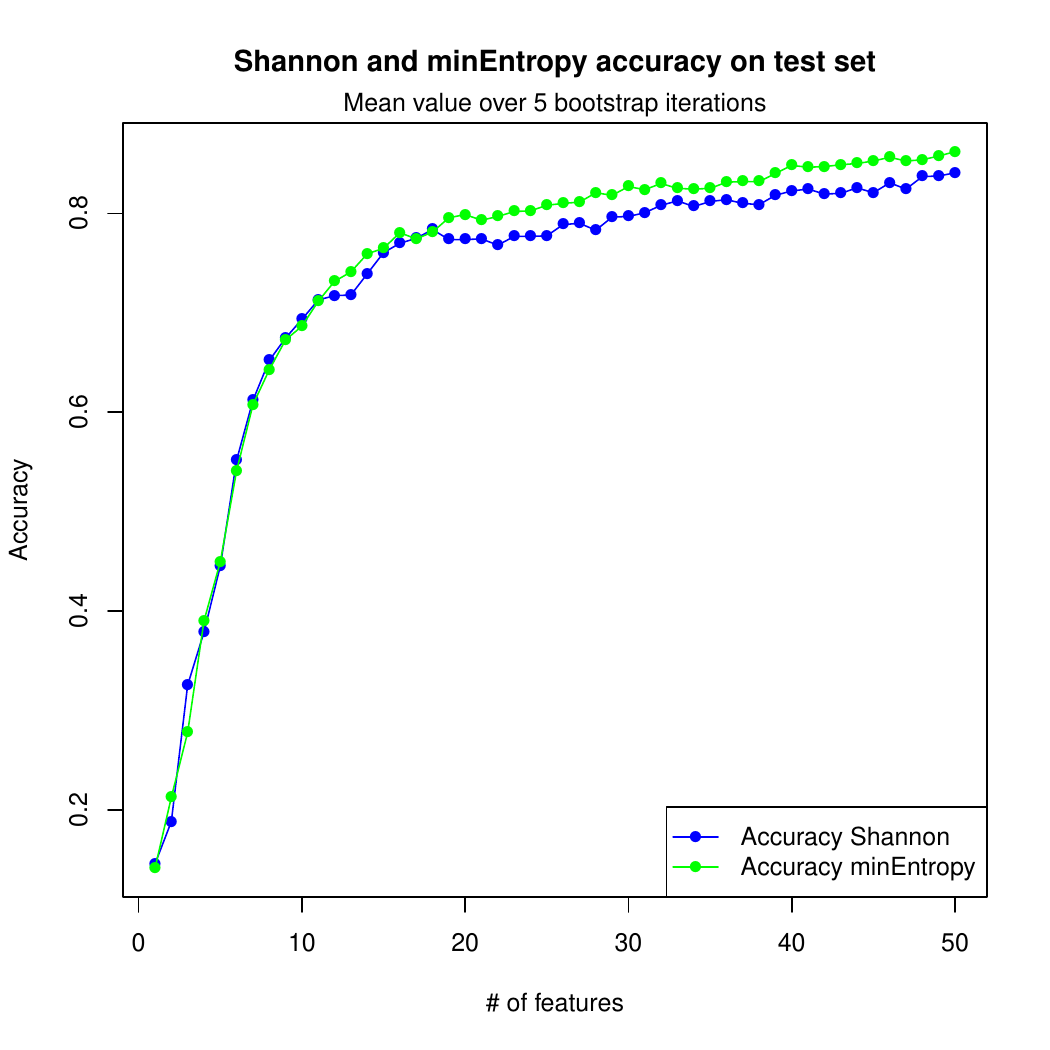}
\vspace{-3mm}
\caption{Accuracy of the ANN and SVM classifiers on the SEMEION dataset}\label{fig:se}
\vspace{-3mm}
\end{figure}

\section{Related works}\label{sec:related}
In the last two decades,  due to the growing interest for machine learning, much research effort has been devoted to the  feature reduction problem, and several 
methods have been proposed. 
In this section we discuss those closely related to our work, namely those which are based on information theory. 
For a more complete overview  we refer to \cite{Bennesar:15:ESA}, \cite{Vergara:14:NCA} and \cite{Brown:12:JMLR}.

The approach most related to our proposal is that of \cite{Brown:12:JMLR} and \cite{Vergara:14:NCA}, which differ from ours in that it uses Shannon entropy instead than R\'enyi min entropy. We have discussed  and compared   their method with ours in the technical body of this paper. 

As far as we know,  in the context of feature selection R\'enyi min-entropy has only been considered by  \cite{Endo:13:CIARP}
(although in \revision{the} experiments they only show results for other  R\'enyi entropies). 
The notion of conditional R\'enyi min-entropy they use, however, is that of \cite{Cachin:97:PhD}, which  formalizes it 
along the lines of conditional Shannon entropy. Namely, 
the conditional min-entropy of $X$ given $Y$ is defined  in  \cite{Cachin:97:PhD} as the expected value of the entropy of $X$ for each given value of $Y$. 
Such definition, however, violates the \emph{monotonicity principle}: knowing the value of  $Y$ may increase the entropy of  $X$ 
 instead of decreasing it. It is clear, therefore, that basing a method on this notion of entropy could lead to strange results.

 
 Two  key concepts that have been widely used are \emph{relevance} and \emph{redundancy}. 
Relevance refers to the importance   for the classification of the  feature under consideration at time $t$, $f^t$, and it is in general modeled as $I(C;f^t)$, where $I$ is the Shannon mutual information. 
Redundancy represents how much the information of $ f^t$ is already covered by $S$, and it is often modeled as $I(S;f\textsuperscript{t})$.
 In general, we want to maximize relevance and minimize redundancy. 
 
One of the first algorithms ever implemented was the MIFS algorithm proposed by  \cite{Battiti:94:TNN},  based on a greedy  strategy. At the first step   it selects 
$f\textsuperscript{1} = \argmax_{f_i \in F}I(C;f_i)$, and at  step $t$ it selects  
$ f\textsuperscript{t} = \argmax_{f_i \in F\setminus S\textsuperscript{t-1}}[I(C; f_i)-\beta \sum_{f_s \in S\textsuperscript{t-1}}I(f_s; f_i)]
$ where $\beta$ is a parameter that controls the weight of the redundancy part.

The mRMR approach (redundancy minimization and relevance maximization) proposed by \cite{Peng:05:TPAMI}  is based on the same strategy as MIFS. However the redundancy term is now substituted by its mean over the  elements of the subset $S$ so to avoid its value to grow when new attributes are selected.
 
A common issue with these two methods is that they do not take  into account the conditional mutual information $I(C;f\textsuperscript{t} \mid S)$ for the choice of the next feature to be selected. As a consequence,  it may happen that a feature $f\textsuperscript{t}$ has a high correlation with some other feature in the set $S$ of features already chosen, but,  if $I(C;f\textsuperscript{t})$ is high,   $f\textsuperscript{t}$ may still be  selected despite the fact that it is highly redundant.
 
More recent algorithms involve the ideas of joint mutual entropy $I(C;f_i,S)$ (JMI, \cite{Bennesar:15:ESA}) and conditional mutual entropy $I(C;f_i \mid S)$ (CMI, \cite{Fleuret:04:JMLR}).
The   step for choosing the next feature with JMI 
  is 
$f\textsuperscript{t}=\argmax_{f_i \in F\setminus S\textsuperscript{t-1}}\big\{min_{f_s \in S\textsuperscript{t-1}}I(C;f_i,f_s)\big\}$,
while with CMI is 
$f\textsuperscript{t}=\argmax_{f_i \in F\setminus  S\textsuperscript{t-1}} \big\{min_{f_s \in S\textsuperscript{t-1}}I(C;f_i \mid f_s)\big\}$.
In both cases the already selected features are taken into account one by one when compared to the new feature   $f\textsuperscript{t}$.
In  \cite{Yang:99:AIDA} the following correlation between JMI and CMI  was proved: 
\begin{equation*}
I(C;f_i,S) = H(C)-H(C \mid S)+H(C \mid S)-H(C \mid S) 
= I(C;S)+I(C;f_i \mid S).
\end{equation*}

\section{Conclusion and Future Work}\label{sec:conclusion}
We have proposed a method for feature selection based on a notion of  conditional R\'enyi min-entropy. 
 Although our method is  in general incomparable with  the corresponding one  based on Shannon entropy, 
in the experiments we  performed it turned out that our methods always achieved better results. 
 
 As future work, we plan to compare our proposal with other information-theoretic methods for feature selection. In particular, we plan to investigate the application of other notions of entropy which are the state-of-the-art  
 in  security and privacy, like the notion of $g$-vulnerability~\cite{Alvim:12:CSF}, which seems promising for its flexibility and capability to represent a large spectrum of possible classification strategies.

%
%
%
\bibliographystyle{splncs04}
\bibliography{main}
\end{document}